\documentclass[twoside]{article}

\usepackage{lineno,hyperref,amssymb}
\usepackage{microtype}
\usepackage{graphicx,amsmath}
\usepackage{subfigure}
\usepackage{booktabs} 

\usepackage{times}
\usepackage{graphicx}
\usepackage{latexsym}
\usepackage{graphicx}
\usepackage{enumitem,algorithm2e}
\usepackage{amssymb,amsmath,amsthm}
\usepackage{color,mathtools}
\newtheorem{thm}{Theorem}
\newtheorem{lem}[thm]{Lemma}

\definecolor{sarandonga}{rgb}{0.65,0.35,0}

\usepackage[accepted]{aistats2021}
\begin{document}

%

%

\twocolumn[

\aistatstitle{Rank aggregation for non-stationary data streams}

\aistatsauthor{ Ekhine Irurozki  \And Aritz Perez \And Jesus Lobo \And Javier del Ser }

\aistatsaddress{ BCAM, Telecom Paris \And BCAM \And  Tecnalia \And Tecnalia } ]

\begin{abstract}
We consider the problem of learning over non-stationary ranking streams. The rankings can be interpreted as the preferences of a population and the non-stationarity means that the distribution of preferences changes over time. 
Our goal is to learn, in an online manner, the current distribution of rankings. The bottleneck of this process is a rank aggregation problem. 

We propose a generalization of the Borda algorithm for non-stationary ranking streams. Moreover, we give bounds on the minimum number of samples required to output the ground truth with high probability. Besides, we show how the optimal parameters are set. Then, we generalize the whole family of weighted voting rules (the family to which Borda belongs) to situations in which some rankings are more \textit{reliable} than others and show that this generalization can solve the problem of rank aggregation over non-stationary data streams. 

\end{abstract}

\section{INTRODUCTION}






In the rank aggregation problem, a population of voters cast their preferences over a set of candidates by providing each a ranking of these candidates. The goal is to summarize this collection of rankings $S$ with one single ranking $\pi$. A classical formulation of rank aggregation is to find the Kemeny ranking, the ranking that minimizes the number of discrepancies with $S$. Unfortunately, finding the Kemeny ranking is NP-hard and usually the --polynomial-time-- Borda ranking is preferred. 

When the rankings are distributed according to realistic models, i.e., the Mallows distribution, Kemeny is known to be the maximum likelihood estimator of one model parameter and Borda a high-quality, low-cost approximation to the same parameter. Therefore, besides in Optimization~\cite{Coppersmith:2010}, rank aggregation has been studied in Computational Social Choice~\cite{Brandt2016} and Machine Learning~\cite{Lu2011,xia2019learning} extensively, generally in a batch setting. 

In this paper, we challenge the standard batch learning setting and consider stream learning~\cite{Gam2010, Krempl2014}, which refers to learning problems for which the data is continuously generated over time. This feature imposes severe computational complexity constraints on the proposed learning algorithms. For instance, in contrast to batch learning, in stream learning scenarios the data no longer can be completely stored, and the learning algorithms must be computationally efficient (linear or even sub-linear), to the extreme of operating close to real-time. These computational complexity constraints usually lead to models that are incrementally updated with the arrival of new data.

Besides, another major challenge when dealing with stream learning arises when the source producing stream data evolves over time, which yields to non-stationary data distributions. This phenomenon is known as concept drift (CD) \cite{gama14}. In the stream learning scenarios with CD, the models have to be able to adapt to changes (drifts) in the data distribution (concept). We focus on this scenario that we call \textit{evolving preferences}.

\paragraph{Motivation}

In this paper, we assume that the data is given as a stream of rankings over $n$ items. Rankings are i.i.d. to the Mallows model (MM), an analogous to the Gaussian distribution defined over the permutation space: MM are parametrized by (i) a modal ranking (representing the consensus of the population), and (ii) a parameter controlling the variance. Therefore, a stream of rankings with concept drift is naturally modelled by a sequence of MM with different modal rankings, i.e., an evolving MM. The goal is to develop an estimate of the current modal ranking of the evolving MM with low computational complexity and arbitrarily small error, that is able to deal with the CD phenomenon.

\paragraph{Contributions}
There are three main contributions in this paper. 
\begin{itemize}

    \item We adapt the Borda algorithm to the stream learning scenario in which the distribution of the data changes, SL with CD. We denote this generalization of Borda unbalanced Borda (uBorda).
    \item We theoretically analyze uBorda and provide bounds on the number of samples required for recovering the last modal ranking of the evolving MM, in expectation and with arbitrary probability. Moreover, we show how to set an optimal learning parameter. 
    \item We generalize the Weighted Tournament solution and Unweighted Tournament solutions to handle situations in which some voters are more \textit{trusted} than others. We denote this setting as \textit{unbalanced voting}. 
\end{itemize}

\paragraph{Related work}
Rank aggregation has been studied in in Optimization~\cite{Coppersmith:2010}, Computational Social Choice~\cite{Brandt2016} and Machine Learning~\cite{Lu2011,xia2019learning}.
In the last decade, rank aggregation has been studied in on-line environments. In particular, \cite{Yasutake2012} proposed a ranking reconstruction algorithm based on the pairwise comparisons of the items with near optimal relative loss bound. In \cite{ailon2014improved}, two efficient algorithms based on sorting procedures that predict the aggregated ranking with bounding maximal expected regret where introduced. In \cite{DBLP:conf/icml/Busa-FeketeHS14} an active learner was proposed that, querying items pair-wisely, returns the aggregate ranking with high probability. 

Evolving preferences have been considered from an axiomatic perspective by modeling voters preferences with Markov Decision Processes~\cite{parkes2013dynamic}, in fairness~\cite{freeman2017fair} and matching~\cite{hosseini2015matching}. In~\cite{tal2015study}, the authors analyze the behavior of the preferences of a population of voters under the plurality rule. In contrast to the previous works, in this work we assume a probabilistic setting using a MM that changes over the time for modeling the evolving preferences of a population.

This paper is organized as follows: Section~\ref{sec:preliminaries} includes the preliminaries of this paper. 
Section~\ref{sec:uborda} presents the generalization of Borda and its theoretical guarantees. 
Section~\ref{sec:voting_rules} introduces the adaptation of a whole family of voting rules to the case in which some voters are trusted more. 
Section~\ref{sec:exper} shows the empirical evaluation of the algorithms for the rank aggregation on stream learning with concept drift. 
Finally, Section~\ref{sec:conclusions} summarizes the main contributions of the work.

\section{PRELIMINARIES AND NOTATION} \label{sec:preliminaries}

Permutations are a bijection of the set $[n]=\{1, 2, \ldots, n\}$ onto itself and are represented as an ordered vector of $[n]$. The set of permutations of $[n]$ is denoted by $S_n$. Permutations will be denoted with the Greek letters $\sigma$ or $\pi$ and represent rankings, meaning that the ranking of element $i \in [n]$ in $\sigma$ is denoted by $\sigma(i)$. Rankings are used to represent preferences and we say that item $i$ is preferred to item $j$ when it has a lower ranking, $\sigma(i)<\sigma(j)$.
 
One of the most popular noisy models for permutations is the Mallows model (MM)~\cite{gMallows,mallows}. The MM is an exponential location-scale model in which the location parameter (modal ranking) is a ranking denoted by $\pi$, and the scale (or concentration parameter) is a non-negative real number denoted by $\theta$. In general, the probability of any ranking $\sigma \in S_n$ under the MM can be written as:
\begin{equation}
    p(\sigma) = \frac{\exp(-\theta \cdot d(\sigma, \pi))}{\psi(\theta)},
\end{equation}
where $\psi(\theta)$ is a normalization factor that depends on the location parameter $\theta$, and has closed form expression.

A random ranking $\sigma$ drawn from a MM centered at $\pi$ and with concentration parameter $\theta$ is denoted as $\sigma \sim MM(\pi,\theta)$. The expected rank of item $i$ for $\sigma \sim MM(\pi,\theta)$ is denoted by $\mathbb{E}_{\pi}[\sigma(i)]$ , or simply by $\mathbb{E}[\sigma(i)]$ when it is clear from the context. To model preferences, the selected distance $d(\cdot,\cdot)$ for permutations is the Kendall's-$\tau$ distance (for other choice of distances see~\cite{Irurozki2016b,Vitelli2018}). This distance counts the number of pairwise disagreements in two rankings,
\begin{equation}
d(\sigma, \pi) = \sum_{i<j} \mathbb{I} [ (\sigma(i)-\sigma(j))*(\pi(i)-\pi(j)) < 0 ],
\end{equation}
where $\mathbb{I}$ is the indicator function.

Given a sample of rankings, the maximum likelihood estimation (MLE) of the parameters of a MM is usually computed in two steps. First, the MLE of the modal ranking $\pi$ is obtained, which corresponds to the Kemeny ranking~\cite{gMallows} (see the following sections for further details). Then, given the MLE of the modal ranking, the MLE of the concentration parameter $\theta$ can be computed numerically.

\subsection{Modeling evolving preferences: Evolving Mallows model}\label{sec:driftmm}

In this paper, we assume that the data is a stream of preferences, i.e., a sequence of rankings and that our proposed algorithm has to update its estimate every time a new ranking is received. The rankings in the sequence may not all come from the same distribution, so in the next lines, we introduce a natural way of modeling preferences that evolve dynamically over time.

Let $\sigma_t$ for $t\geq 0$ be the $t-th$ ranking in the given sequence, being $\sigma_0$ the actual ranking. We assume that the stream of rankings satisfy $\sigma_t \sim MM_t(\pi_t,\theta_t)$, for $t\geq 0$. This sequence of models is denoted the evolving Mallows model (EMM).

When two consecutive models, $MM_{t+1}(\pi_{t+1},\theta_{t+1})$ and $MM_t(\pi_t,\theta_t)$, differ on their modal ranking, $\pi_{t+1} \neq \pi_t$, we say that at time $t$ there has been a drift in the modal ranking of the model. Under the Mallows models, the drifts in the modal ranking can be interpreted as a change in the consensus preference of a population of voters. When the concentration parameter changes over time, the drifts represent a change in the variability of the preferences of a population around the consensus ranking. For the rest of the paper, we focus on the scenario with drifts in the modal ranking.

It is possible to have two different perspectives when we refer to the type of drift. Frequently, drifts can be classified as gradual and abrupt in terms of speed, being abrupt when a change happens suddenly between two concepts, and gradual when there is a smooth transition between both concepts. In this work, we have considered abrupt drifts in what refers to the speed of change (rankings generated from the old concept disappear suddenly and the new ones appear), and gradual drift in what refers to the changes in the order preference of the labels (the Kendall's-$\tau$ distance between consecutive modal rankings is small).

\section{UNBALANCED BORDA FOR STREAMING DATA}
\label{sec:uborda}

In this section, we present a generalization of the well-known Borda count algorithm to rank aggregation for the context of stream data with concept drift. We want to point out that this approach is generalized in Section~\ref{sec:voting_rules}, (1) for the settings in which each ranking in the sample is not equally relevant (e.g., the agents have different reliabilities), and (2) for all the voting rules in the family of Weighted voting rules (or C2 family)~\cite{Brandt2016}. 

The Borda algorithm~\cite{Borda1781} is one of the most popular methods for aggregating a sample of rankings, $S$, into one single ranking that best represents $S$. Borda ranks the items $[n]$ by their \textit{Borda score} increasingly, where the Borda score $B(i)$ is the average ranking of each item $i$, for $i=1,...,n$:
$$B(i) =1/|S| \sum_{t\in S}  \sigma_t(i),$$ where $|S|$ correspond to the length of the set of rankings $S$. 

Borda is the de-facto standard in the applied literature of rank aggregation. Firstly, it has a computational complexity of $\mathcal{O}(n\cdot (|S| + \log n))$. Secondly, it is guarantied to be a good estimator of the modal ranking when the samples are i.i.d. according to a MM: it requires a polynomial number of samples with respect to $n$ to return the modal ranking of the MM with high probability~\cite{Caragiannis2013}. In general, it is a $5$-approximation to the Kemeny ranking~\cite{Coppersmith:2010}, which, unfortunately, has been shown to be NP-hard to compute~\cite{Dwork:2001:RAM:371920.372165}. 

We propose unbalanced \textit{Borda} (\textit{uBorda}) a generalization of Borda for the rank aggregation problem in the context of streaming preferences. 
In uBorda, each ranking $\sigma_t$ has an associated weight $w(\sigma_t)$ that is proportional to the relevance of the ranking in the sample. The uBorda scores correspond to the weighted average of the given rankings. In other words, uBorda is equivalent to replicating each ranking $\sigma_t$ $w(\sigma_t)$ times and applying Borda.

In stream learning, votes arrive sequentially, at time-stamp $t$ rank $\sigma_t$ is received, being $t=0$ the most recent ranking of a possible infinite sequence of votes. In order for uBorda to adapt to the concept drifts in the sequence, we propose to weight the ranking $\sigma_t$ by $\rho^t$ for a given parameter $\rho \in [0,1]$. This choice leads to the following Borda score for each item~$i$.

\begin{equation}
\begin{split}
B(i) \propto \sum_{t>=0} \rho^t \cdot \sigma_t(i)
\label{eq:uBordaScore}
\end{split}
\end{equation}

Intuitively, using this specific weights the voting rule \textit{pays more attention} to recent rankings since $\rho^t$ exponentially decreases as the antiquity $t$ of the ranking $\sigma_t$ increases. The uBorda score can be incrementally computed as 
${ B(i) \propto \sigma_0(i) + \rho\cdot B_1(i)} $ 
where $B_1(i)$ denotes the previous uBorda score,  computed using $\sigma_t$ for $t\geq 1$. Thus, in this streaming scenario, the uBorda scores can be computed incrementally in linear time $\mathcal{O}(n)$, and the Borda algorithm has a computational time complexity of $\mathcal{O}(n\log n)$, in the worst case.

By setting $\rho=1$ we recover the classic Borda. When $0\leq \rho<1$, uBorda can be seen as a Borda algorithm that incorporates a forgetting mechanism to adapt the ranking aggregation process to drifts in streaming data.

In subsequent sections, we will theoretically analyze the properties of uBorda for ranking aggregation in streaming scenarios when the rankings are i.i.d. according to an EMM. 
First, next lemma provides some known intermediate result regarding the expected value of rankings obtained from a MM. 

\begin{thm}[\cite{Fligner1988}]
Given a sample $S$ of permutations i.i.d. according to a $MM(\pi,\theta)$, Borda algorithm outputs a consistent estimator of $\pi$.
\label{lem:consistent}
\end{thm}

Following, we show a proof sketch to present the main ideas behind this result and introduce some new notation (see \cite{Fligner1988} for further details). The authors argue that, on average and for $\pi(i)<\pi(j)$, Borda\footnote{Note that, with a slight abuse in the notation, we are using $\sigma$ indistinctly for a ranking and for a random variable distributed according to a MM.} will rank $i$ above $j$  iff $\mathbb{E}_\pi[\sigma(i)] < \mathbb{E}_\pi[\sigma(j)] $. They reformulated the expression $\mathbb{E}_\pi [\sigma(j)] < \mathbb{E}_\pi [\sigma(i)] $ in a more convenient way:

\begin{equation}
\begin{split}
\Delta_{ij} &= \mathbb{E}_\pi[\sigma(j)]-\mathbb{E}_\pi[\sigma(i)] \\ 
& =  \sum_{\{\sigma: \sigma(i)<\sigma(j)\}} (\sigma(j) - \sigma(i)) (p(\sigma) - p(\sigma\tau)),
\end{split}
\label{eq:cij}
\end{equation}
where $\tau$ is an inversion of positions $i$ and $j$. Clearly, the proof of Theorem \ref{lem:consistent} is equivalent to showing that $\Delta_{ij}>0$  for every $i,j$ such that $\pi(i)<\pi(j)$. To see this, note that for the set of rankings $\{\sigma: \sigma(i)<\sigma(j)\}$ it holds that $(\sigma(j) - \sigma(i))>0$. When $p$ is a $MM(\pi,\theta)$, due to the strong unimodality of MMs, we can also state that $(p(\sigma) - p(\sigma\tau))>0$. Thus, given a MM with parameters $\pi$ and $\theta$ we have that $\Delta_{ij}>0$ for every $i,j$ such that $\pi(i)<\pi(j)$.

In the following Lemma, we provide an intermediate result that relates the expected rankings of two MMs, $MM(\sigma,\theta)$ and $MM(\sigma\tau,\theta)$, where $\tau$ represents an inversion of positions $i$ and $j$.
\begin{lem} 
Let $\sigma$ a ranking distributed according to $MM(\pi,\theta)$. 
Let $\tau$ be an inversion of $i$ and $j$ so that  $d(\pi\tau,\pi)=1$.
Using the definition on Equation~\eqref{eq:cij}, we can easily see that $\mathbb{E}_{\pi}[\sigma]$ and $\mathbb{E}_{\pi\tau}[\sigma]$ are related as follows:

\begin{equation}
\begin{split}
\mathbb{E}_{\pi\tau}[\sigma(i)] &= \mathbb{E}_{\pi}[\sigma(j)] = \mathbb{E}_{\pi}[\sigma(i) ] + \Delta_{ij} \\
\mathbb{E}_{\pi\tau}[\sigma(j)] &= \mathbb{E}_{\pi}[\sigma(i)] \label{thm:cij_drift}
\end{split}
\end{equation}
\end{lem}

\subsection{Sample complexity for returning $\pi_0$ on average}

In this section, we analyze theoretically uBorda for ranking aggregation using stream data distributed according to a EMM. We consider a possibly infinite sequence of rankings for which a drift in the modal ranking has occurred $m$ batches before, where the current modal ranking is $\pi_0$. We bound the number of batches since the last drift that uBorda needs to recover the current modal ranking $\pi_0$ on average.

\begin{thm}
Let $\pi(i)<\pi(j)$, and let $\tau$ be an inversion of $i$ and $j$ so that $d(\pi\tau,\pi)=1$. Let $\sigma_t$ be a (possibly infinite) sequence of rankings generated by sampling an EMM, such that  ${\sigma_t \sim MM(\pi,\theta)}$ for $t \geq m $ and ${\sigma_t \sim MM(\pi\tau,\theta)}$ for $m > t$.
Then, uBorda returns the current ranking  $\pi_0=\pi\tau$ in expectation when $$m > \log_\rho 0.5.$$
\label{thm:borda_consistent}
\end{thm}

\begin{proof}
The uBorda algorithm ranks the items in $k \in [n]$ w.r.t. the uBorda score $B(k)= \sum_{t\geq 0} \rho^t \sigma_t(k)$. Thus, uBorda recovers the current ranking $\pi_0$ if and only if the expression ${\mathbb{E}[\sum_{t\geq 0} \rho^t \sigma_t(i)] < \mathbb{E}[\sum_{t\geq 0} \rho^t \sigma_t(j)]}$ is satisfied for every pair $i,j$.

\begin{align}
\begin{split}
\mathbb{E}[\sum_{t \geq 0} \rho^t \sigma(i)] &= \sum_{t \geq m} \rho^t \mathbb{E}_{\pi}[\sigma(i)] + \sum_{m> t} \rho^t \mathbb{E}_{\pi\tau} [\sigma(i)] \\
	&= \sum_{t \geq 0} \rho^t  \mathbb{E}_{\pi}[\sigma(i)]  + \sum_{m> t } \rho^t \Delta_{ij} \\
\mathbb{E}[\sum_{t\geq 0} \rho^t \sigma(j)] &= \sum_{t\geq m} \rho^t \mathbb{E}_{\pi}[\sigma( j )] + \sum_{m> t} \rho^t \mathbb{E}_{\pi\tau} [\sigma( j )] \\
&	= \sum_{t \geq 0} \rho^t  \mathbb{E}_{\pi}[\sigma(i)] + \sum_{t \geq m} \rho^t \Delta_{ij}
\end{split}
\end{align}

Therefore, the uBorda algorithm recovers the current ranking $\pi_0$ in expectation if and only if the next inequality holds:
\begin{align}
\sum_{t<m } \rho^t <  \sum_{t \geq m} \rho^t.
\label{eq:uBorda_exp}
\end{align}
Thus, given $m>0$, we can satisfy the inequality of Equation ~\eqref{eq:uBorda_exp} by selecting and appropriate value of $\rho$ as follows:
\begin{align}
m >  \log_\rho 0.5. 
\end{align}
\end{proof}
Note that the right hand side expression decreases as $\rho$ decreases, for $\rho \in (0,1]$, and in the limit, ${\lim_{\rho\rightarrow 0}log_\rho 0.5= 0}$.  
The intuitive conclusion is that as $\rho$ decreases uBorda is more reactive, that is, in expectation it needs less samples to accommodate to the drift. In other words, when $\rho$ decreases uBorda quickly forgets the old preferences and focuses on the most recently generated ones, so it can quickly adapt to recent drifts. One might feel tempted of lowering $\rho$ to guarantee that uBorda recovers the last modal ranking in expectation. However, as we decrease $\rho$ the variance on the uBorda score (see Equation \eqref{eq:uBordaScore}) increases and, thus the variance of the ranking obtained with uBorda also increases. In the next section, we will show how the confidence of our estimated modal ranking decreases as $\rho$ decreases. Besides, we will provide a lower bound on the number of samples required by uBorda for recovering from a drift.

\subsection{Sample complexity for returning $\pi_0$ with high probability}

In this section we consider a possibly infinite sequence of rankings and denote by $\pi_0$ the current consensus of the model, which has generated the last $m$ rankings. We bound the value of $m$ that uBorda needs to return $\pi_0$ with high probability in Theorem~\ref{thm:sample_complexity}.

We present an intermediate result in Lemma~\ref{lem:conf_inter}, where we bound the difference that with high probability will be between the uBorda score and its expected value. For this intermediate result we consider that there is no drift in the sample.

\begin{lem}
\label{lem:conf_inter}
Let $\sigma_r,\sigma_{r+1},...,\sigma_{s}$ be $m=s-r+1$ rankings i.i.d. distributed according to $MM(\pi,\theta)$.
In the absence of drifts, the absolute difference between the uBorda score for item $i$, $(1-\rho)\cdot \sum_{t= r}^s \rho^t \sigma_t(i)$, and its expectation $\mathbb{E}[(1-\rho)\sum_{t= r}^s \rho^t \sigma_t(i)]$ is smaller than 

\begin{equation}
\begin{split}
\epsilon_r^s=   (n-1)(1-\rho)\cdot \sqrt{\frac{(\rho^{2r}-\rho^{2s})}{2 \cdot (1-\rho^2)} \cdot \log \frac{2}{\delta}}
\label{ref:deviation}
\end{split}
\end{equation}
with, at least a probability of $1-\delta$, 
$$P \Bigg(\Big|\sum_{t=r}^{s} \rho^t (\sigma_t(i) -\mathbb{E}[\sigma_t(i)]) \Big| \leq \epsilon_r^s \Bigg) \geq 1-\delta$$
\end{lem}

The proof, which is omitted, uses the Hoeffding's inequality and sums of series.



Now, we are ready to consider EMM and give a lower bound on the number of samples required for recovering from a drift with high probability. The lower bound is given as a function of the concentration of the underlying distribution $\theta$ and parameter $\rho$.

\begin{thm}\label{thm:sample_complexity}
Let $\pi(i)<\pi(j)$ and let $\tau$ be an inversion of $i$ and $j$ so that $d(\pi\tau,\pi)=1$. 
Let $\sigma_t$ for $t\geq 0$ be a (possibly infinite) sequence of rankings generated by sampling an evolving MM, such that for $t \geq m $ then $\sigma_t \sim MM(\pi,\theta)$  and for $t < m $ then $\sigma_t \sim MM(\pi\tau,\theta)$.
The number of samples that uBorda needs to returns the current modal ranking $\pi_0=\pi\tau$ in expectation with probability $1-\delta$ is at least
\begin{equation}
    \begin{split}
        m > \log_\rho \Bigg ( \frac{-(1-\rho)^2}{\sqrt{1-\rho^2}} \frac{n \sqrt{0.5\log \delta^{-1}}}{\Delta_{ij}} +0.5 \Bigg )
    \end{split}
\end{equation}
where $\Delta_{ij}$ is defined in Equation~\eqref{eq:cij}. 
\end{thm}

\begin{proof}
Following the idea in Theorem \ref{thm:borda_consistent}, the algorithm uBorda returns the current ranking $\pi_0=\pi\tau$ in expectation when the expected uBorda score (see Equation \eqref{eq:uBordaScore}) of element $i$ is greater than of element $j$ as

\begin{equation}
\begin{split} 
\sum_{t=m}^{\infty} & \rho^t \mathbb{E}_{\pi}[\sigma(j)] + \sum_{t=0}^{m-1} \rho^t \mathbb{E}_{\pi\tau}[\sigma(j)]  < \\
	& \sum_{t=m}^{\infty}  \rho^t \mathbb{E}_{\pi}[\sigma(i)] + \sum_{t=0}^{m-1} \rho^t \mathbb{E}_{\pi\tau}[\sigma(i)] 
\end{split}
\label{eq:sample1}
\end{equation}

which is based on Equation~\eqref{eq:cij}. Equation~\eqref{eq:sample1} would be an accurate measure for the sample complexity if the expected value did not deviate at all from the sum. However, according to Lemma~\ref{lem:conf_inter} we can upper bound the difference between the uBorda score as shown in  Lemma~\ref{lem:conf_inter}. Next, we make use of the definition in Equation~\eqref{ref:deviation} to define the deviations of the uBorda score before the drift, $\epsilon_m^{\infty}$, and after the drift, $\epsilon_0^{m-1}$ and let $ \epsilon_0^{\infty}$.
Note that the next inequality holds:
\begin{equation}
\epsilon_m^{\infty} +  \epsilon_0^{m-1} = \epsilon_0^{\infty}
\label{eq:sumerror}
\end{equation}

Therefore, we can say that with probability $1-\delta$ we have recovered from a drift when the $m$ satisfies 
\begin{equation}
\begin{split} 
\sum_{t=m}^{\infty} & \rho^t \mathbb{E}_{\pi}[\sigma(j)]  + \epsilon_m^{\infty} + \sum_{t=0}^{m-1} \rho^t \mathbb{E}_{\pi\tau}[\sigma(j)] + \epsilon_0^{m-1} < \\
&	\sum_{t=m}^{\infty}  \rho^t \mathbb{E}_{\pi}[\sigma(i)] - \epsilon_m^{\infty} + \sum_{t=0}^{m-1} \rho^t \mathbb{E}_{\pi\tau}[\sigma(i)] - \epsilon_0^{m-1}  \\
\Rightarrow & \Delta_{ij} \sum_{t=m}^{\infty}  \rho^t  + 2 \epsilon_m^{\infty}  < \Delta_{ij} \sum_{t=0}^{m-1} \rho^t - 2 \epsilon_0^{m-1} 
\end{split}
\end{equation}

Therefore, we can we can state that with probability $1-\delta$  uBorda will recover $\pi\tau$  when the following expression holds:

\begin{equation}
\Delta_{ij}  \sum_{t=0}^{m-1} \rho^t - 2\epsilon_0^{m-1}  > \Delta_{ij}  \sum_{t=m}^{\infty} \rho^t + 2 \epsilon_m^{\infty}
\end{equation}

After some algebra and using the result in Equation~\eqref{eq:sumerror}, we obtain the lower bound for $m$
\begin{equation}
\begin{split}
\sum_{t=0}^{m-1} \rho^t \Delta_{ij} - \sum_{t=m}^{\infty} \rho^t \Delta_{ij}  & > 2\epsilon_m^{\infty} +  2\epsilon_0^{m-1} > 2 \epsilon_0^{\infty}\\
\frac{\rho^m-1}{\rho-1} - \frac{\rho^m}{1-\rho}  & > \frac{\epsilon_0^{\infty}}{\Delta_{ij}} \\
\rho^m &> \frac{\epsilon_0^{\infty}(\rho-1)}{\Delta_{ij}}+0.5 \\
m > \log_\rho \Bigg ( \frac{(\rho-1)(1-\rho)}{\sqrt{1-\rho^2}} & \frac{n \sqrt{0.5\log \delta^{-1}}}{\Delta_{ij}} +0.5 \Bigg )\\
\end{split}
\label{eq:final_bound}
\end{equation}
which concludes the proof. 
\end{proof}

\subsection{Optimal value for $\rho$}

In this section we provide a practical criteria for choosing the forgetting parameter $\rho$ of uBorda. The next result shows how to compute the value of $\rho$ for recovering from a drift with high probability after receiving $m$ rankings.

\begin{thm}\label{thm:rho_opt}
Let $m$ be the number of rankings received after the last concept drift. uBorda recovers the true ranking with probability $(1-\delta)$ by using the forgetting parameter value $\rho* = \arg\max_\rho f(\rho, m)$ that can be found numerically, where

\begin{equation}
\begin{split}
        f  (\rho, & m) =\\
    & \frac{2\rho^m-1}{\rho-1}   
    \Bigg (\sqrt{\frac{\rho^{2m}}{1-\rho^2}} \frac{1-\rho}{\rho^m} + \sqrt{\frac{\rho^{2m}-1}{\rho^2-1}} \frac{\rho-1}{\rho^m-1} \Bigg ).
\end{split}
\end{equation}
\end{thm}

\begin{proof}
The derivation starts by the assumption that Theorem~\ref{thm:sample_complexity} is satisfied. This is equivalent to saying that Equation~\eqref{eq:final_bound} is holds with equality. By rewriting $\epsilon_0^{m-1}$ and $\epsilon_m^{\infty}$, the following is equivalent. 

\begin{equation}
    \begin{split}
\sum_{t=0}^{m-1} \rho^t \Delta_{ij} & - \sum_{t=m}^{\infty} \rho^t \Delta_{ij} = 2\epsilon_m^{\infty} +  2\epsilon_0^{m-1}  \\
& =  (\hat\epsilon_m^{\infty} +  \hat\epsilon_0^{m-1} ) 2n\sqrt{0.5 \log \delta^{-1}}
    \end{split}
    \label{}
\end{equation}

By reordering all the terms containing $\rho$ and $m$, we obtain the expression of function $f(\rho,m)$ in the left hand side,

\begin{equation}
    \begin{split}
        f(\rho, m) = \frac{2\rho^m-1}{\rho-1} \cdot  \frac{1}{\hat\epsilon_m^{\infty} +  \hat\epsilon_0^{m-1}}
        =& \frac{2n\sqrt{0.5 \log \delta^{-1}}}{\Delta_{ij}}.
    \end{split}
    \label{eq:opt_rho1}
\end{equation}

By reordering again, we get the expression for $\delta$

\begin{equation}
    \begin{split}
        \delta = \exp \Bigg ( \Big (\frac{2n\sqrt{2}}{f(\rho^*,m) \Delta_{ij} }\Big)^2 \Bigg).
    \end{split}
\end{equation}

The minimum probability of error $\delta$ is reached for the maximum value of $f(\rho, m)$. Since function $f(\rho, m)$ is concave for $\rho\in (0,1)$, numerical methods obtain a arbitrarily good approximation to its maximum. 

\end{proof}

The previous result shows how to incorporate expert knowledge into the uBorda, i.e., maximize the probability of recovering the ground true ranking after $m$ rankings. Moreover, it has a probability of success as $1-\delta$. An more basic approach to include this kind of expert knowledge is to set a window of size $m$ and consider only the last $m$ rankings. In this case, the probability of success can be found with classic quality results for Borda~\cite{Caragiannis2013}. However, our proposed uBorda has different advantages over traditional window approaches. First, uBorda allows handling and infinite sequence of permutations. This means that for every pair of items that have not suffered a drift there is an infinity number of samples available. Second, this analysis can be adapted to more general settings. In other words, this paper considers the particular situation in which recent rankings are more relevant than old ones but there are situations in which a subset of the rankings are more relevant than other for different reasons. For example, because some voters provide rankings that are more trustworthy than others. In this case, uBorda can be used as a rank aggregation procedure in which expert agents have a larger weight. We consider this general setting in the next section. Moreover, we extend the analysis to two whole families of voting rules in Section~\ref{sec:voting_rules}.

\section{GENERALIZING VOTING RULES}\label{sec:voting_rules}
So far, we considered that recent rankings are more important than old rankings. 
In this section, we consider the general case of the setting in which some voters are more \textit{trusted} than others and therefore, the rankings of the former are more important than the rankings of the latter. In particular, we are given a set of rankings $S= \{\sigma_v: v \in \mathcal{V}\}$ representing the preferences of a set of voters $\mathcal{V}$, where each voter $v \in \mathcal{V}$ has a weight $w_v \in \mathbb{R}^+$ representing our confidence in $\sigma_v$.
We refer to this setting as  \textit{unbalanced voting}. There are similar contexts in crowd learning scenarios in which the weights of each voter is related to its reliability~\cite{Karger2011}. 

Besides Borda, there are many other voting rules in the literature. The most relevant voting rules are those in the Unweighted Tournament Solutions family (C1 voting rules), and the Weighted Tournament Solutions family (C2 voting rules~\cite{Brandt2016}\footnote{Borda belongs to C2 voting rules.}). In this section, we generalize the C2 voting rules. Essentially, we generalize the \textit{frequency matrix} upon which all the rules in these families are defined. 


The weighted frequency matrix, $N$, is the following summary statistic of the sample of rankings $S$.

\begin{equation}
N_{ij} = \sum_{v \in \mathcal{V}}  w_v \mathbb{I} [\sigma_v(i)<\sigma_v(j)],
\label{eq:nmatrix}
\end{equation}
where $w_v$ is the weight of the preference $\sigma_v$. The standard frequency matrix has weights $w_v=1$ for $v \in \mathcal{V}$. The family of C2 functions are given as a function of the \textit{majority margin} matrix, wich is computed using the frequency matrix. The generalized majority matrix, $M$, is computed using $N$ as follows:
\begin{equation}
M_{ij} = N_{ij} - N_{ji} = 2N_{ij} - \sum_{v \in \mathcal{V}} w_v,
\label{eq:mmatrix}
\end{equation}
Intuitively, $M$ counts the difference on the weighting votes that prefer $i$ to $j$ and those that prefer $j$ to $i$. Again, by setting $w_v=1$ for $v \in \mathcal{V}$ in Equation~\eqref{eq:nmatrix}, the original definition of majority margin is recovered. 

The most renowned voting rule in the C2 family is the Kemeny ranking, which aggregates the votes (rankings) in the sample to the permutation that maximizes the agreement, i.e., the permutation that minimizes the sum of the distances to the sample of rankings. The Kemeny ranking can be formulated using the majority margins matrix in Equation~\eqref{eq:mmatrix} as follows
\begin{equation}
\begin{split}
\sigma_K =\arg\min_{\sigma \in S_n}  \sum_{\substack{i,j \in [n] : \sigma(i) > \sigma(j)}} M_{ij}.
\end{split}
\label{eq:kemeny}
\end{equation}

The Kemeny ranking has been shown to maintain several interesting properties, such as being a median permutation of the sample of rankings, and the MLE of the modal ranking of the sample when the rankings are i.i.d. according to a Mallows model~\cite{gMallows}. Unfortunately, the problem of computing the Kemeny ranking given a sample of rankings $S$ has been shown to be NP-hard~\cite{Dwork:2001:RAM:371920.372165}.

Other interesting members of the family of the C2 family build upon the majority margin defined in Equation~\eqref{eq:mmatrix}. Worth highlighting are a pairwise query algorithm that returns the Kemeny ranking with high probability~\cite{Braverman2008}, the ranked pairs method~\cite{Tideman1987} and a $4/3$-approximation of the Kemeny ranking~\cite{Ailon2008}. 

The denotation of the Weighted voting rules comes from the fact that the majority margin matrix can be seen as the adjacency matrix of the a graph where the nodes are the items being ranked. The C1 family or Unweighted Voting Rules family uses the unweighted version of the graph defined by the majority margin matrix $M$. As a conclusion, we state that the unbalanced voting can be adapted to 2 complete families of voting rules.

\section{EXPERIMENTS}\label{sec:exper}
In this section, we provide empirical evidences of the strengths of uBorda to deal with the ranking aggregation problem in streaming scenarios. We illustrate that uBorda procedure can be used to find the most probable ranking of an evolving strong unimodal probability distribution over rankings. Moreover, we show how to select the value of $\rho$ when we expect uBorda to recover from the drift with high probability after a given number of samples.

\subsection{Rank aggregation for dynamic preferences} \label{sec:exp_uBorda}

This section analyses the performance of uBorda as a rank aggregation algorithm using streaming data of rankings. The experiments are evaluated with a synthetic dataset using an EMM. 

We show in the next lines how to generate the sample from an EMM of length $l= T*n*(n-1)/2$ given by $\pi_t$ for $t=l-1, ..., 0$ with drifts that occurs periodically every $T$ time stamps, i.e. $\pi_{l-T*i}\neq \pi_{l-T*i-1}$ for $i=1,...,n\cdot(n-1)/2-1$. The drifts are incremental satisfying that the Kendall's-$\tau$ distance between $\pi_{l-1}$ and $\pi_0$ is the maximum ($n\cdot(n-1)/2$) and for each drift the distance between $\pi_{t+1}$ and $\pi_t$ is one. This artificial scenario simulates abrupt, small drifts in terms of rankings that transforms a given starting ranking $\pi_l$ into its reverse $\pi_0$, e.g. for $n=5$ and $\pi_l=(12345)$ we have that $\pi_0=(54321)$. 

The concentration parameter $\theta_t$ is constant for ${t=l-1,...,0}$. It has been chosen so that the expected distance of the generated rankings to the consensus ranking $\pi_t$ is $1/3$ of the expected distance under the uniform distribution. This choice is not casual: for a very large choice of $\theta$, the samples are very close in Kendall's-$\tau$ distance to each other and to the modal ranking, making it an easy case both practically and theoretically. For $\theta=0$, the resulting sample is uniform. None of the cases satisfies a practical scenario. However, we believe that this setting is of interest in both theory and practice. 

The final sample consists of sampling an evolving MM whose sequence of consensus is $\Pi$. In other words, for each  $\pi_t$ we define the distribution $MM(\pi_t, \theta)$, and generate the sample $ \{\sigma^t_1, \ldots, \sigma^t_m\}$ by sampling that model ${\sigma^t_i \sim MM(\pi_t, \theta) }$ \cite{Irurozki2016b}. 

The evaluation process is done following the test-then-train strategy, a common approach in stream learning in which and instance $\sigma_t^i$ is generated, then used for evaluation and then feed to the training model. The error of uBorda is measured as Kendall's-$\tau$ distance $d(\bar\pi_t^i,\pi_t)$, where $\bar\pi_t^i$ is the uBorda ranking after sampling $\sigma_t^i$ and $\pi_t$ the current consensus ranking. Once the evaluation is done, the new instance $\sigma_t^i$ is appended to the sample, so it can be considered by uBorda in the next iteration of the process.

Note that we can handle a possibly infinite stream of rankings. Interestingly, the Borda counts of this infinite stream of rankings is stored with linear space complexity.

\begin{figure*}[t]
\centering
\includegraphics[width=.5\textwidth]{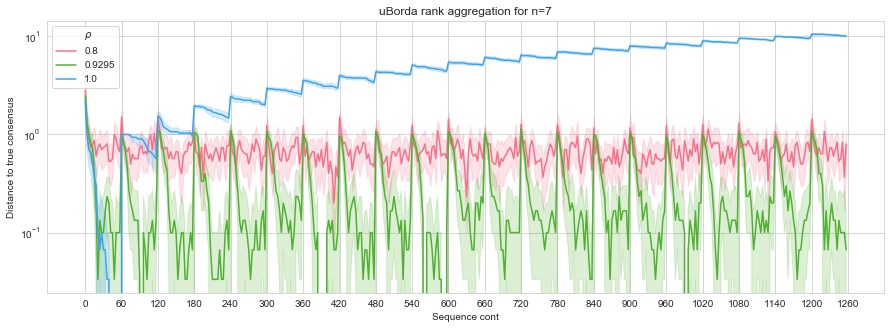}~
\includegraphics[width=.5\textwidth]{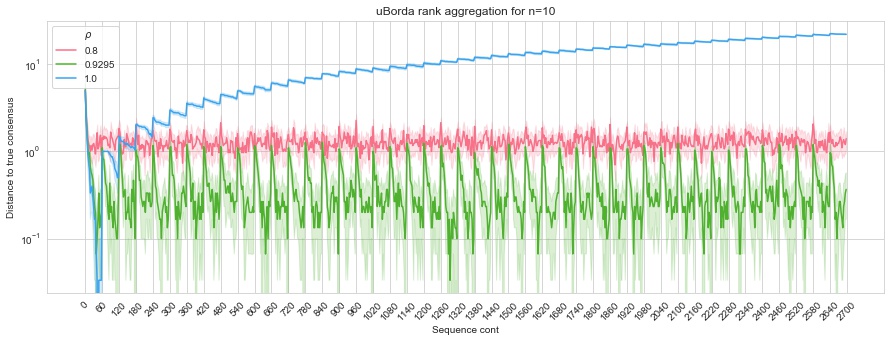}
\caption{Error in the estimated uBorda consensus after a new ranking in the sequence is received. The sequence is ordered in the x axis and there is a drift at each vertical line. The ranking size is $n=7$ (up) and $n=10$ (bottom).}
\label{fig:wbor}
\end{figure*}

The results of $n=7$ and $n=10$ are shown in Figure~\ref{fig:wbor}, where the X-axis orders the sequence of rankings chronologically and the Y-axis shows the error $d(\bar\pi_t^i, \pi_t)$.

Simulating the inclusion of expert knowledge in the experiment, we expect the rank aggregation to recover form the drifts after $m=20$ samples. Theorem~\ref{thm:rho_opt} shows that finding the value for $\rho$ that minimizes the probability of error for a given number o samples is done by solving a convex optimization problem. For our choice of $m=20$ the optimal $\rho$ is $0.9295$. For comparison, we use different fading factors $\rho \in  \{0.8,0.9295, 1\}$, each corresponding to a different line in the plot. 

For each parameter configuration the results are run 30 times and the average are shown with a stroke line and the .95 confidence interval ($\lambda=0.05$) as a shadow.

The first evaluations after a drift occurs (after each vertical grid-line), the error increases for every choice of $\rho$. As expected, as the number of rankings of the same distribution increases, the error tends to decrease. However,  differences is the value of  $\rho$ cause critical differences in the behaviour of uBorda.

When $\rho=0.9295$ uBorda has the most accurate results. As shown in Theorem~\ref{thm:rho_opt}, this value maximizes the probability of recovering from a drift in 20 samples among the parameter values considered. Moreover, after these 20 samples, it recovers the ground true ranking with error smaller than 0.1. 

Choosing $\rho=0.8$ makes uBorda forget quicker the previous permutations and this can lead to a situation in which too few of the last permutations are considered to estimate the consensus. The more chaotic behavior of the smallest value of $\rho=0.8$ (remind the logarithmic scale for the Y-axis) is related to this phenomenon in which few permutations are contributing in the estimation of the consensus, i.e., uBorda is aggregating a small number of rankings. 

Finally, for $\rho=1$ (when no fading factor is considered, which in turn is equivalent to using the standard Borda) has an increasing error. This is because it is assuming that the last and the first permutations seen are equally important, and the consensus of the population does not change in time. This is equivalent to a standard online Borda algorithm.

\section{CONCLUSIONS}\label{sec:conclusions}

In this paper, we have considered a novel scenario for rank elicitation which assumes online learning scenarios in which the distribution modeling the preferences changes as time goes by. Under this realistic prism, we have studied two well-known ranking problems: rank aggregation and label ranking.

Our main contribution is to generalize the Weighted Tournament solutions to the situation in which some voters are more reliable than others. We call this scenario \textit{unbalanced voting} and allow arbitrary \textit{reliability} values for the voters.

We argue that the unbalanced voting scenario is can be applied to the stream learning with concept drift. We denote as uBorda the version of Borda in which the voters have different weights and show that it can be computed efficiently, handling a possibly infinite sequence of rankings in linear space and quasi-linear time complexity. A thouh analysis leads to several contributions. First, we bound the number of samples required by uBorda to output the current ground truth modal ranking with high probability. Second, we show how to include expert knowledge to the uBorda. Finally, we have shown its efficiency in several empirical scenarios. 

In this paper, we raise the question of dynamic preferences, and this idea opens several interesting research lines. For example, we can analyze different voting rules in unbalanced voting contexts. Moreover, these voting rules can be used in different problems from stream learning since there many machine learning problems that handle situations in which some experts are more trustworthy than others. Moreover, we plan on considering the design of new aggregation algorithms for evolving Mallows model under different distances. Similar questions have already been considered in an offline setting for the Cayley \cite{Irurozki2018} and Hamming~\cite{Irurozki2014a} distances. Moreover, other ranking models such as Plackett-Luce~\cite{luce59} or Babington Smith~\cite{critchlow91} distributions are worth considered under this point of view. 

Finally, as a future work we plan to extend this analysis of unbalanced voting rules to different scenarios in which some rankings are more relevant than other. This scenario arises frequently in general machine learning problems when some agents are known to be more trustworthy than others. In combinatorial optimization problems, we consider the use of uBorda to speed up the process of finding an optimum by setting the weights of each solution proportionally to the fitness function of the given solution.


\bibliography{./mendeley}

\end{document}